\documentclass{article} 
\usepackage[final]{colm2026_conference}

\usepackage{microtype}
\usepackage{hyperref}
\usepackage{url}
\usepackage{booktabs}

\usepackage{lineno}
\usepackage{url}          
\usepackage{amsmath}
\usepackage{xcolor}
\usepackage{blkarray} \usepackage{bigstrut}

\usepackage{subcaption}   
\usepackage{graphicx}     

\usepackage{tikz}
\usetikzlibrary{shapes.geometric, arrows.meta, positioning, calc}

\providecolor{vectorblue}{RGB}{224,238,250}
\providecolor{resultgreen}{RGB}{226,242,226}
\providecolor{symbolgray}{RGB}{235,235,235}

\usepackage[table]{xcolor}
\usepackage{array}
\usepackage{collcell}
\usepackage{xfp}

\definecolor{heatlow}{HTML}{FCFBFD}
\definecolor{heathigh}{HTML}{6A51A3}

\newcommand{\HeatMaxTint}{55} 
\newcommand{\heat}[1]{%
  \edef\val{\fpeval{min(100,max(0,#1))}}%
  \edef\pct{\fpeval{round(\HeatMaxTint*\val/100,0)}}%
  \edef\colorspec{heathigh!\pct!heatlow}%
  \expandafter\cellcolor\expandafter{\colorspec}#1%
}
\newcolumntype{H}{>{\collectcell\heat}c<{\endcollectcell}}

\usepackage{enumitem}

\usepackage{subcaption}
\usepackage{float}
\usepackage{mathtools}
\usepackage{amssymb}
\usepackage{amsfonts}
\usepackage{amstext}
\usepackage{amsthm}
\usepackage{latexsym}
\usepackage{stmaryrd}
\usepackage{colortbl}
\allowdisplaybreaks[4]

\usepackage{pgfplots}
\usepackage{pgfplotstable}
\pgfplotsset{compat=1.18}

\usepackage[capitalise,nameinlink]{cleveref} 
\crefformat{section}{#2\S #1#3}
\Crefformat{section}{#2\S #1#3}
\crefname{theorem}{Thm.}{Thm.}
\crefname{lemma}{Lem.}{Lem.}
\crefname{definition}{Def.}{Def.}
\crefname{example}{Ex.}{Ex.}

\usepackage{comment}

\usepackage{inconsolata}
\usepackage[nott]{newtx}

\newtheorem{theorem}{Theorem}[section]
\newtheorem{lemma}[theorem]{Lemma}
\newtheorem{corollary}[theorem]{Corollary}

\theoremstyle{definition}

\newtheorem{example}[theorem]{Example}

\usepackage{etoolbox}

\DeclareMathOperator{\cs}{cs}
\DeclareMathOperator{\idx}{idx}

\newcommand{\R}{\mathbb{R}}

\newcommand{\softmax}{\mathop{\textnormal{softmax}}}

\newcommand{\mat}[1]{\mathbf{#1}}

\newcommand{\qproj}{\mat{W}^{\text{Q}}}
\newcommand{\kproj}{\mat{W}^{\text{K}}}
\newcommand{\vproj}{\mat{W}^{\text{V}}}
\newcommand{\LN}{{\bf \textnormal{LN}}}
\newcommand{\dhid}{{d_{\textnormal{hid}}}}
\newcommand{\dkey}{{d_{\textnormal{key}}}}
\newcommand{\dtok}{{d_{\textnormal{tok}}}}
\newcommand{\etok}{{\mat{e}_{\textnormal{tok}}}}

\newcommand{\vocab}{{\cal V}}

\definecolor{softgreen}{rgb}{0.88, 1, 0.88}
\definecolor{softorange}{rgb}{1, 0.88, 0.7}
\definecolor{softred}{rgb}{1, 0.7, 0.7}

\definecolor{darkblue}{rgb}{0, 0, 0.5}
\hypersetup{colorlinks=true, citecolor=darkblue, linkcolor=darkblue, urlcolor=darkblue}

\title{Concise One-Layer Transformers Can Perform Dictionary Lookup}

\author{
  Anton Hampe\\
  Yale University\\
  \href{mailto:anton.hampe@yale.edu}{\tt anton.hampe@yale.edu}
  \And
  Lena Strobl\\
  Yale University\\
  \href{mailto:lena.strobl@yale.edu}{\tt lena.strobl@yale.edu}
  \And
  Dana Angluin\\
  Yale University\\
  \href{mailto:dana.angluin@yale.edu}{\tt dana.angluin@yale.edu}
  \And
  Robert Frank\\
  Yale University\\
  \href{mailto:robert.frank@yale.edu}{\tt robert.frank@yale.edu}
}

\begin{document}

\ifcolmsubmission
\linenumbers
\fi

\maketitle

\begin{abstract}
While transformers have proven enormously successful in a range of tasks, their fundamental properties as models of computation are not well understood.
This paper contributes to the study of the expressive capacity of transformers, focusing on their ability to perform the fundamental computational task of looking up a key in a dictionary.
We prove that concise 1-layer transformers (i.e., with a polylog bound on the product of the number of heads, the embedding dimension, and precision) can perform this task when the key and value are represented in the same input position, but, under the $p$-bit rational precision assumption, cannot perform it when the keys and values are represented in different input positions.
Concise 2-layer transformers can perform the task in the second case.
Experimentally, we find an alignment between what we have proven can be computed by concise transformers and what can be practically learned.

\end{abstract}
\section{Introduction}
\label{sec:introduction}
The transformer architecture for neural networks has had a revolutionary impact on the development and deployment of large language models.
There are many papers with empirical assessments of the performance of transformers for various tasks, but less is known about their in-principle limits as models of computation.

\paragraph{Expressivity and Trainability.}
This paper, like most current theoretical work, addresses \emph{expressivity}, that is, whether there exists any set of weights that would allow a transformer to perform a given task.
Expressivity is a pre-condition of \emph{trainability}, that is, whether there is an optimization procedure that could efficiently find some set of weights that would allow a transformer to perform a given task, given examples of it.
Of course, expressivity does not ensure trainability. For example, \citet{chiang-cholak-2022-overcoming} show that a multi-layer transformer encoder can compute the parity of a binary string, but to date, successful training has required carefully crafted multi-stage curricula \citep{kim2025transformersprovablysolveparity}.
Corresponding to our theoretical results on expressivity, we provide experimental results illustrating the trainability of the task we consider under various conditions.

\paragraph{Universality and Conciseness.}
Expressivity has been addressed by universality results. \citet{yun2020aretransformers} prove that an arbitrary function can be approximated by a transformer, although their construction requires unbounded depth.
\citet{kajitsuka2024aretransformers} strengthen this result, showing that transformers with a single self-attention layer are universal function approximators, however, in general their construction requires a transformer of size exponential in the input length.
In these constructions, the transformer depends on the input length, and we refer to this situation as a \emph{family of transformers}, by analogy to a family of circuits.
We view the growth rate of these general solutions as impractically large, and seek \emph{concise} solutions. We characterize transformer size in terms of their \emph{c-size}: the product of the embedding dimension, the number of heads, and precision. We consider concise transformers to be those whose c-size grows as a polylogarithmic function of its input length. 
We study the expressive power of a concise family of one-layer transformers with a single head as the basis for more complex architectures with multiple heads, multiple layers, looping, chain of thought, diffusion, and other enhancements.

\paragraph{Dictionary lookup.}
The task we consider is \emph{dictionary lookup} (or \emph{associative lookup}) in which the input is a dictionary consisting of a finite list of key/value pairs together with a target key, and the desired output is the value that is paired with the target key in the dictionary.
This task is a fundamental building block of many algorithms, and
understanding the conditions under which it can be performed by a transformer will help us understand whether and how transformers can perform more complex algorithms.

The keys and values are tokens from a vocabulary ${\cal V}$ with $|{\cal V}| = N$, and there is a token embedding that maps each token to a vector.
We study two different input representations of a dictionary.
\begin{itemize}
    \item[] {\bf Same position.} The two vectors representing a key/value pair are concatenated and placed in one input position.
    \item[] {\bf Consecutive positions.} The two vectors representing a key/value pair are placed in consecutive input positions.
\end{itemize}

In the same position case, each key/value association is local to one input position, so one attention layer can retrieve the corresponding value directly.
In the consecutive positions case, each association is distributed across two input positions.

\paragraph{Our Contributions.}
\begin{itemize}

\item For the same position case, in \cref{sec:same_position_1_layer_construction} we construct a concise family of 1-layer hard attention transformers with a single head, embedding dimension $O(1)$ and $O(\log N)$ bits of precision that can perform dictionary lookup.  We show in \cref{sec:temperature-simulations} that this construction is effective in transformers with soft attention with the addition of temperature. 

\item
For the consecutive positions case, in \cref{sec:consecutive_positions_1_layer_bounds} we prove (non-concise) lower and upper bounds of $\Theta(N \log N)$ on the c-size of 1-layer softmax attention transformers that can perform dictionary lookup. We also 
construct a concise family of 2-layer hard attention transformers  with one head, embedding dimension $O(1)$ and $O(\log N)$ bits of precision that can perform dictionary lookup.

\item 
For the same position case, in \cref{section:empirical-results} we find that randomly initialized small transformers can be trained with gradient descent to perform dictionary lookup, though at embedding dimensions larger than in the theoretical construction.
For the consecutive positions case, we find that 1-layer transformers struggle to learn dictionary lookup, even for small vocabularies and dictionaries.
\end{itemize}

\paragraph{Related work.}
\citet{strobl_what_10.1162/tacl_a_00663} survey theoretical results on the expressiveness of transformers as models for the recognition of formal languages.
\citet{yang2025transformercookbook} give a compendium of theoretical techniques useful for constructing transformers to perform various tasks. Our work also builds on previous theoretical work showing lower bounds on 1-layer transformers for certain tasks.

Communication complexity has been a fruitful approach to lower bounds.
\citet{peng2024limitationstransformerarchitecture} 
consider the task of function composition and use
communication complexity to demonstrate that transformers with a single attention layer of fixed size are unable to perform function composition over functions of sufficiently large domain.
\citet{sanford2024onelayertransformersfailsolve} focus on the capacity of 1-layer transformers to perform the induction head task \citep{elhage2021amathematicalframework,olsson2022incontextlearninginductionheads}.
They use communication complexity to prove a lower bound 
of $\Omega(n)$ on the transformer size necessary to perform this task, where $n$ is the input length.

An alternative approach using VC-dimension to show lower bounds on 1-layer transformers is presented by
\citet{kozachinskiy2025strassenattentionunlockingcompositional}. They prove a lower bound $\max\{h, d, P\} \geq n^{\Omega(1)}$ on 1-layer transformers to compute function and relation composition, where $h$ is the number of heads, $d$ is the embedding dimension, and $P$ is the number of parameters in the feedforward layer.

\section{Preliminaries}
\label{sec:preliminaries}

\subsection{Conventions}
\label{subsec:conventions}
For any positive integer $n$, $[n]$ denotes the set of integers $\{0,1,\ldots,n-1\}$.
We use the notation $(\mathbb{R}^d)^+$ for sequences (of arbitrary finite positive length) with elements from $\mathbb{R}^d$.
If $\mathbf{X} \in (\R^d)^+$ has length $n$, then for $i \in [n]$, $\mathbf{X}[i] \in \R^d$ refers to element $i$ of the sequence, with 0-based indexing.
A function $g \colon (\R^d)^+ \to (\R^d)^+$ is \emph{length-preserving} if it maps an input sequence of length $n$ to an output sequence of the same length $n$.

We consider an extension of $\softmax$ to ${\softmax}_{T}$ to incorporate a \emph{temperature} $T$, where the effect of lower temperature is to magnify the differences in attention weights for two different attention scores.
Standard $\softmax$ then corresponds to temperature $T = 1.0$.
For $\mat{x} \in \R^n$ and real-valued $T > 0$,
\[{\softmax}_T(\mat{x})[i] = \frac{\exp(\mat{x}[i]/T)}{\sum_{j \in [n]} \exp(\mat{x}[j]/T)}.\]

\subsection{Dictionary Lookup}

For the dictionary lookup task, there is a fixed vocabulary $\vocab$ of tokens, where $N = |\vocab|$.
We assume an arbitrary ordering $\tau_0,\tau_1,\ldots,\tau_{N-1}$ of the elements of $\vocab$, and a function $\idx$ such that $\idx(\tau_j) = j$.
This ordering is used in the matrices that embed and unembed tokens.
A \emph{dictionary} of length $n$ is a finite sequence of key/value pairs $(k_0,v_0), \ldots, (k_{n-1},v_{n-1})$ where $(k_i,v_i) \in \vocab \times \vocab$ for $i \in [n]$, such that no two pairs have the same key, that is, if $i \neq j$ then $k_i \neq k_j$.
Note that we must have $n \le N$.

A \emph{dictionary lookup problem} consists of a dictionary and a \emph{target key} $k^* \in \vocab$, a token such that for some $i$, $k^* = k_i$.
The correct output for a dictionary lookup problem is the unique token $v_i$ such that $k^* = k_i$.
Note that the correct output does not depend on the order of pairs in the dictionary.
For ease of understanding, we leave aside the variant of dictionary lookup in which the target key does not occur in the dictionary, though it should not be a difficult extension.

The dictionary lookup task is similar to the induction heads task considered by \citet{sanford2024onelayertransformersfailsolve}.
In their definition of the standard induction heads task, the input is a sequence of symbols from a fixed vocabulary, and the output at each position is the symbol that follows the most recent previous occurrence of the symbol at that position, or a null symbol if there is no previous occurrence. For example, for $X=baebcabebdea$, the final symbol is $a$, which occurs twice previously, and the final output is $b$, the symbol following its most recent previous occurrence.
In this task, there is no distinction between keys and values, and there may be multiple occurrences of the symbol being sought.

An important difference is that in the induction heads task, the vocabulary size is a constant independent of the input length, and constructions typically use a one-hot encoding of symbols.
In the dictionary lookup task, the vocabulary may be as large as the input length, which emphasizes the properties of the vocabulary embedding.
In particular, a one-hot encoding of the vocabulary cannot yield a concise family of transformers for dictionary lookup.

\subsection{Transformer Models}
\label{subsec:transformer_models}

We follow the definition of a transformer model given by \citet {yang2025transformercookbook}.
A \emph{transformer} of embedding dimension $d$ is a mapping from a non-empty input sequence to an output.
It has a length-preserving \emph{input embedding}, which maps a sequence of inputs to an element of $(\R^d)^+$,
a sequence of length-preserving attention and feedforward layers that map $(R^d)^+$ to $(R^d)^+$,
and an \emph{output mapping}, which maps from $(\R^d)^+$ to an output, which in our case is a single element of $\R^d$.

\paragraph{Input Embedding.}
\label{subsubsec:input_embedding}
We specialize the input embedding for the two input presentations we consider.

\begin{figure}[tbp]
  \centering
  
  \begin{subfigure}{\linewidth}
    \centering
    \resizebox{\linewidth}{!}{\begin{tikzpicture}[
  node distance=1.8cm,
  symbol/.style={
    rectangle,
    rounded corners=2pt,
    draw=black!70,
    fill=symbolgray,
    minimum width=26mm,
    minimum height=9mm,
    align=center,
    font=\normalsize
  },
  positional/.style={
    rectangle,
    rounded corners=2pt,
    draw=green!45!black,
    fill=resultgreen,
    minimum width=32mm,
    minimum height=11mm,
    align=center,
    font=\normalsize
  },
  flow/.style={
    ->,
    >=Stealth,
    semithick,
    draw=black!75
  }
]

  \node[symbol] (k0)    at (-7.2, 0) {$k_0$};
  \node[symbol] (v0)    at (-3.6, 0) {$v_0$};
  \node[symbol] (k1)    at ( 0.0, 0) {$k_1$};
  \node[symbol] (v1)    at ( 3.6, 0) {$v_1$};
  \node[font=\large] (dots1) at (5.7, 0) {$\dots$};
  \node[symbol] (kn)    at ( 7.8, 0) {$k_{n-1}$};
  \node[symbol] (vn)    at (11.4, 0) {$v_{n-1}$};
  \node[symbol] (kstar) at (15.0, 0) {$k^*$};

  \node[positional] (pos_k0)    at (-7.2, 2.2) {$w(k_0) + p(0)$};
  \node[positional] (pos_v0)    at (-3.6, 2.2) {$w(v_0) + p(1)$};
  \node[positional] (pos_k1)    at ( 0.0, 2.2) {$w(k_1) + p(2)$};
  \node[positional] (pos_v1)    at ( 3.6, 2.2) {$w(v_1) + p(3)$};
  \node[font=\large] (dots2)    at ( 5.7, 2.2) {$\dots$};
  \node[positional] (pos_kn)    at ( 7.8, 2.2) {$w(k_{n-1}) + p(2n-2)$};
  \node[positional] (pos_vn)    at (11.4, 2.2) {$w(v_{n-1}) + p(2n-1)$};
  \node[positional] (pos_kstar) at (15.0, 2.2) {$w(k^*) + p(2n)$};

  \draw[flow] (k0) -- (pos_k0);
  \draw[flow] (v0) -- (pos_v0);
  \draw[flow] (k1) -- (pos_k1);
  \draw[flow] (v1) -- (pos_v1);
  \draw[flow] (kn) -- (pos_kn);
  \draw[flow] (vn) -- (pos_vn);
  \draw[flow] (kstar) -- (pos_kstar);

\end{tikzpicture}}
    \caption{Consecutive position embedding}
    \label{fig:consecutive_position}
  \end{subfigure}

  \vspace{1em} 

  \begin{subfigure}{\linewidth}
    \centering
    \resizebox{\linewidth}{!}{\begin{tikzpicture}[
  node distance=1.8cm,
  symbol/.style={
    rectangle,
    rounded corners=2pt,
    draw=black!70,
    fill=symbolgray,
    minimum width=28mm,
    minimum height=9mm,
    align=center,
    font=\normalsize
  },
  concatenated/.style={
    rectangle,
    rounded corners=2pt,
    draw=green!45!black,
    fill=resultgreen,
    minimum width=32mm,
    minimum height=11mm,
    align=center,
    font=\normalsize
  },
  flow/.style={
    ->,
    >=Stealth,
    semithick,
    draw=black!75
  }
]

  \node[symbol] (pair0)    at (-5.5, 0) {$(k_0, v_0)$};
  \node[symbol] (pair1)    at (-1.8, 0) {$(k_1, v_1)$};
  \node[font=\large] (dots1) at (1.2, 0) {$\dots$};
  \node[symbol] (pairn)    at ( 4.2, 0) {$(k_{n-1}, v_{n-1})$};
  \node[symbol] (pairstar) at ( 7.9, 0) {$(k^*, 0)$};

  \node[concatenated] (w0)    at (-5.5, 2.2) {$w((k_0, v_0))$};
  \node[concatenated] (w1)    at (-1.8, 2.2) {$w((k_1, v_1))$};
  \node[font=\large] (dots2)  at ( 1.2, 2.2) {$\dots$};
  \node[concatenated] (wn)    at ( 4.2, 2.2) {$w((k_{n-1}, v_{n-1}))$};
  \node[concatenated] (wstar) at ( 7.9, 2.2) {$w((k^*, 0))$};

  \draw[flow] (pair0) -- (w0);
  \draw[flow] (pair1) -- (w1);
  \draw[flow] (pairn) -- (wn);
  \draw[flow] (pairstar) -- (wstar);

\end{tikzpicture}}
    \caption{Same position embedding}
    \label{fig:same_position}
  \end{subfigure}

  \caption{Comparison of input embedding methods}
  \label{fig:embedding_comparison}
\end{figure}

For the consecutive positions case, the input alphabet is $\vocab$ (a key or value), and we assume that the input order is $k_0,v_0,\ldots,k_{n-1},v_{n-1},k^*$.
The input length $\ell(n)$ is $2n+1$.
The \emph{word embedding} $\mat{w}$ maps $\vocab$ to $\R^d$, and the \emph{position embedding} maps $[\ell(n)]$ to $\R^d$. The \emph{input embedding} $\mat{e}$ maps a sequence of input tokens $x_0,\ldots,x_{\ell(n)-1}$ to a sequence of elements of $\R^d$ by $\mat{e}(x_i) = \mat{w}(x_i) + \mat{p}(i)$.
This presentation is typical when considering dictionary lookup as an end-to-end task.

For the same position case, the input alphabet of the transformer is $(\vocab \times \vocab) \cup \vocab$ (a key/value pair or a key), and we assume the order is $(k_0,v_0),\ldots,(k_{n-1},v_{n-1}),k^*$.
The input length $\ell(n)$ is $n+1$.
In this case, we assume an embedding $\etok$ that maps tokens from $\vocab$ to $\R^{\dtok}$.
The embedding dimension is $d = 2 \dtok$,
and the \emph{word embedding} $\mat{w}$ maps $(k_i,v_i) \in \vocab \times \vocab$ to $\R^d$ by concatenating the vectors $\etok(k_i)$ and $\etok(v_i)$.
The word embedding $\mat{w}$ maps the target key $k^*$ to the concatenation of $\etok(k^*)$ and $\mat{0} \in \R^\dtok$.
In this case, there is no position embedding, so the input embedding $\mat{e}$ is just $\mat{w}$.
This presentation is intended to reflect a situation in which earlier layers of the transformer have constructed a favorable representation of the task.
This is analogous to the second layer of a 2-layer construction to perform the induction heads task, in which the first layer has copied into each position the symbol at the preceding position \citep{sanford2024onelayertransformersfailsolve}.
See also Sec.~\cref{sec:consecutive_positions_2_layer_construction}.

\paragraph{Scaled Dot-Product Softmax Attention Layer.}
\label{subsubsec:attention}
We specialize our definition to one attention head per layer.
A \emph{(scaled dot-product) softmax self-attention layer} with $d$ input/output dimensions and $\dkey$ key/value dimensions is a length-preserving function mapping $\mat{X} \in (\R^d)^+$ to $\mat{Y} \in (\R^d)^+$ parameterized by $\qproj,\kproj \in \R^{\dkey \times d}$ and $\vproj \in \R^{d \times d}$.
For positions $i$ and $j$ in $\mat{X}$ and $\mat{Y}$, the \emph{query} at $i$ is $\mat{Q}[i] = \qproj(\mat{X}[i])$, the \emph{key} at $j$ is  $\mat{K}[j] = \kproj(\mat{X}[j])$, and the \emph{value} at $j$ is $\mat{V}[j] = \vproj(\mat{X}[j])$.
Then we define:
\begin{align*}
  \mat{s}[i,j] &= \frac{\mat{Q}[i] \cdot \mat{K}[j]}{\sqrt\dkey} \\
  \alpha[i,:] &= \softmax(\mat{s}[i,:]) \\
  \mat{Y}[i] &= \sum_{j\in[n]} \alpha[i,j] \, \mat{V}[j] \\
\end{align*}
We call $\mat{s}[i,j]$ the \emph{attention scores} and  $\alpha[i,j]$ the \emph{attention weights}.

To include \emph{residuals} and \emph{pre-norm layer normalization} we would instead set $\mat{Y}[i] = \mat{X}[i] + \LN(\sum_{j\in[n]} \alpha[i,j] \, \mat{V}[j])$, where $\LN$ is a layer normalization function with parameters $\gamma, \beta \in \R^d$.
In Section \ref{sec:temperature-simulations} we also consider using ${\softmax}_T$, softmax with temperature $T$ (\cref{subsec:conventions}), in place of standard softmax.

\paragraph{Average Hard Attention.}
\label{subsubsec:UHAT}
For simplicity, our theoretical constructions use average hard attention instead of softmax attention; we address this difference in our empirical results.
For any vector $\mat{x} \in \R^n$, define $M(\mat{x}) = \{i \in [n] \mid \forall j \in [n], \mat{x}[j] \le \mat{x}[i]\}$ to be the set of indices of the maximal elements of $\mat{x}$.
In an \emph{average hard} attention layer, instead of the softmax computation, the elements of $M(\mat{s}[i,:])$ are each assigned attention weight $1/|M(\mat{s}[i,:])|$ and the other elements are assigned attention weight $0$.
Note that if there is a unique position $j$ with the maximum attention score, it receives weight $1$ and all other positions receive weight $0$, which
enables exact copying of $\mat{V}[j]$. 

\paragraph{Feedforward Layers.}
\label{subsubsec:MLP_layers}
A \emph{feedforward layer} is a mapping $f$ from $\R^d$ to $\R^d$ that can be implemented by a 2-layer multilayer perceptron with ReLU activation.
It is determined by a dimension $\dhid$ and two affine transformations, one from $\R^d$ to $\R^{\dhid}$ and one from $R^{\dhid}$ to $R^d$.
The input $\mat{X}$ is mapped to the output $\mat{Y}$ elementwise, that is, $\mat{Y}[i] = f(\mat{X}[i])$.
A feedforward layer may also have residuals and pre-norm layer normalization.

\paragraph{Output Mapping.}
\label{subsubsec:output_mapping}
In our case, the desired output is an element of $\cal{V}$.
Let $\mat{Y} \in (\R^d)^{\ell(n)}$ be the output of the final layer of the transformer, and $\mat{Y}[\ell(n)-1]$ its final element.
There is an \emph{unembedding matrix} $\mat{U} \in \R^{N \times d}$, and the output of the transformer is taken to be the token $\tau_i \in \vocab$ for the index $i \in [N]$ that maximizes $(\mat{U}\mat{Y}[\ell(n)-1])[i]$.

We are primarily concerned with the number of attention layers a transformer has, so in this paper a \emph{$k$-layer transformer} is one with $k$ attention layers and any constant number of feedforward layers.
We consider encoder-only transformers with no masking; this is equivalent to future masking for 1-layer transformers in which the output depends only on the last position.

\subsection{The Communication Size of Transformers}

As noted in the introduction, we consider \emph{families of transformers}: for every positive integer $n$, there is a transformer $T_n$ that processes input sequences of length at most $n$, and we analyze how the parameters of the transformers change as $n$ increases.

We define a transformer's \textit{communication-size}, henceforth \emph{c-size}, to be the product $hdp$, where $h$ is the number of heads per layer, $d$ is the embedding dimension, and $p$ is the precision (in bits) of the representation of numbers in the transformer.
The notion of c-size reflects the cost of communication, and does not depend on the number of parameters in any feedforward layer.
We present upper and lower bounds on the c-size of 1-layer and 2-layer transformers to perform dictionary lookup, expressed as a function of the maximum input length for the transformer. 

\citet{sanford2023representationalstrengthslimitationstransformers}, \citet{chen2024theoreticallimitationsmultilayertransformer}, and \citet{kozachinskiy2025strassenattentionunlockingcompositional} distinguish between families of transformers of ``small'' ($n^{o(1)}$, subpolynomial) and ``large'' ($n^{\Omega(1)}$, polynomial or more) size.\footnote{\citet{kozachinskiy2025strassenattentionunlockingcompositional} define the size of a 1-layer transformer as $\max\{h, d, P\}$, where $h$ is the number of heads, $d$ is the embedding dimension, and $P$ is the number of parameters in the feedforward layer. Thus c-size includes precision but excludes feedforward parameter count, whereas their measure includes feedforward parameter count but does not explicitly include precision.}
Here we adopt a more stringent notion of conciseness; we say that a family of transformers is \emph{concise} if its c-size is $O(\log^{O(1)} n)$ (poly-logarithmic) in the maximum input length $n$.

\section{Same Position: A Concise Family of 1-Layer Transformers}
\label{sec:same_position_1_layer_construction}

In this section we give a construction of a concise family of 1-layer hard attention transformers for dictionary lookup for the same position case.
We first define the representation of integers we use.

\paragraph{Representing Integers.}
For any positive integer $m$ we consider a representation of elements of $[m]$ using angles in a unit circle.
For $j \in [m]$, define the angle $\theta_m(j) = 2\pi j/m$, and represent $j$ by the 2-element vector
$\cs_m(j) = [\cos(\theta_m(j)),\sin(\theta_m(j))]^\top$.
In this representation, for $j, k \in [m]$ the dot product of $\cs_m(j)$ with $\cs_m(k)$ is $1$ if $j = k$ and is at most $1 - 1/m^2$ if $j \neq k$.
Thus, $O(\log m)$ bits of precision suffice to distinguish two different integers from $[m]$ in this representation.
(See \cref{fig:sincosconstruction} in \cref{circle_graphic} for an example.)

We generalize this representation to include a frequency $f \in [m]$ as follows.
Let $\theta_{m,f}(j)$ be the angle $2\pi f j / m$, and define $\cs_{m,f}(j) = [\cos(\theta_{m,f}(j)),\sin(\theta_{m,f}(j))]^\top$.
Then $\cs_{m,1}$ is the same as $\cs_m$.
In Sections \ref{sec:temperature-simulations} and \ref{section:empirical-results}, we consider representing $j \in [m]$ using a concatenation of vectors $\cs_{m,f}(j)$ for multiple frequencies $f$.

\begin{theorem}
\label{thm:same_position_1_layer}
For any finite vocabulary $\vocab$ of cardinality $N$, there exists a 1-layer transformer that performs dictionary lookup for vocabulary $\vocab$ in the same position case.
The transformer uses average hard attention,
one head, 
embedding dimension $4$, 
$O(\log N)$ bits of precision, no position embedding, no residuals, no layernorm, and no feedforward layer.
The c-size is $O(\log N)$.
\end{theorem}

\begin{proof}
The input is a dictionary problem
$(k_0,v_0),(k_1,v_1),\ldots,(k_{n-1},v_{n-1}),k^*$,
where the keys $k_i$, values $v_i$, and target key $k^*$ are elements of $\vocab$.
The desired output is $v_j$ such that $k_j = k^*$.

We construct the input embedding $\mathbf{X} = [\mathbf{x}_0, \mathbf{x}_1, \dots, \mathbf{x}_{n-1},\mathbf{x}_n]$, as follows.
For $j \in [n]$:
\[
    \mathbf{x}_j 
    = 
    \begin{bmatrix}
    \cs_N(\idx(k_j)) \\
    \cs_N(\idx(v_j))
    \end{bmatrix}
    =
    \begin{bmatrix}
    \cos(\theta_N(\idx(k_j))) \\
    \sin(\theta_N(\idx(k_j))) \\
    \cos(\theta_N(\idx(v_j))) \\
    \sin(\theta_N(\idx(v_j))) \\
    \end{bmatrix}, \text{ and }
\mathbf{x}_n
= \begin{bmatrix}
    \cs_N(\idx(k^*)) \\
    0_{2 \times 1}
\end{bmatrix}
=
\begin{bmatrix}
    \cos(\theta_N(\idx(k^*))) \\
    \sin(\theta_N(\idx(k^*))) \\
    0 \\
    0 \\
\end{bmatrix}.
\]
Thus the embedding dimension is $4$.

We define the matrices $\qproj$, $\kproj$, and $\vproj$ as:
\[
    \qproj = \kproj = \begin{bmatrix}
    \mat{I}_2 & \mat{0}_{2 \times 2}
    \end{bmatrix}, \qquad
    \vproj = \begin{bmatrix}
    \mat{0}_{2 \times 2} & \mat{0}_{2 \times 2} \\
    \mat{0}_{2 \times 2} & 2 * \mat{I}_2 
    \end{bmatrix},
\]
where $\mat{I}_2$ is the $2 \times 2$ identity matrix.
For a vector $\mat{x} \in \R^4$, each of $\qproj \mat{x}$ and $\kproj \mat{x}$ are both the vector of the first two coordinates of $\mat{x}$, and $\vproj \mat{x}$ is the vector of twice the last two coordinates of $\mat{x}$ preceded by two $0$'s.

For the target key $k^* \in \vocab$, the query vector at position $n$ is:
\[
    \mat{q}_n = \qproj \mat{x}_n = \cs_N(\idx(k^*)) = [\cos(\theta_N(\idx(k^*))),\sin(\theta_N(\idx(k^*)))]^\top.
\]

For each position $j \in [n]$, the key vector is
 \[
     \mat{k}_j
     = \kproj \mat{x}_j
     = \cs_N(\idx(k_j)).
 \]
 At position $n$, the key vector is
 \[
     \mat{k}_n
     = \kproj \mat{x}_n
     = \cs_N(\idx(k^*)).
 \]
At position $n$, the attention scores are
 \[
     s_j=\mat{s}[n,j]
     =\frac{\mat{q}_n^\top\mat{k}_j}{\sqrt{\dkey}}.
 \]

We have $s_j = 1/\sqrt{2}$ for those $j$ such that $k_j = k^*$.
There are two such positions: the $j \in [n]$ that holds the dictionary entry $(k_j,v_j)$, and position $n$ itself.
All other $j$ have $s_j < 1/\sqrt{2}$.
Using average hard attention, the two positions $j$ with maximum attention score each get attention weight $1/2$, and all other positions get attention weight $0$.
The output of the attention layer at position $n$ is:
\[
\mat{h} = \sum_{j=0}^{n} \alpha_j \vproj \mathbf{x}_j.
\]
This is $[0,0,\cos(\theta_N(\idx(v_j))),\sin(\theta_N(\idx(v_j)))]^\top$ such that $j \in [n]$ and $k_j = k^*$.

We define the unembedding matrix $\mat{U} \in \R^{N \times 4}$ to consist of the $N$ rows $[0,0,\cos(\theta_N(r)),\sin(\theta_N(r))]$ for $r \in [N]$.
Applying it to $\mat{h}$, the attention output at position $n$, we have $\mat{\hat{v}} = \mat{U} \mat{h}$.
The final output of the transformer is the index $r$ where $\mat{\hat{v}}$ is maximized, which is $r = \idx(v_j)$ such for the $j \in [n]$ such that $k_j = k^*$, that is, the correct output for dictionary lookup.
\end{proof}
An example of the construction in the proof of \cref{thm:same_position_1_layer} is given in \cref{appendix:example_of_same_position_1_layer}.

It might be objected that the transformer constructed in the proof of \cref{thm:same_position_1_layer} is too simplified to show that practical transformers can perform the task of dictionary lookup because it uses average hard attention instead of softmax attention and lacks features like residuals, layernorms, and a feedforward layer.
Using techniques described by \citet{yang2025transformercookbook}, we can add residuals, layernorms, and a feedforward layer to the model in such a way that the same output is obtained.\footnote{These techniques may multiply the embedding dimension by a constant, which leaves our conciseness results unaffected.} 

\begin{corollary}
\label{cor:1_layer_with_residuals_feedforward_layernorm}
For any finite vocabulary $\vocab$ of cardinality $N$, there exists a 1-layer transformer that performs dictionary lookup for vocabulary $\vocab$ in the same position case.
The transformer has average hard attention,
no position embedding,
residuals, pre-norm layernorms, a feedforward layer,
one head, 
embedding dimension $O(1)$, and
$O(\log N)$ bits of precision.
The c-size is $O(\log N)$.
\end{corollary}

For the difference between softmax attention and average hard attention, we could replace average hard attention by softmax attention with temperature.
It is well known that softmax attention with temperature approaches average hard attention as the temperature approaches $0$.
\citet{yang2024simulatinghardattentionusing} give quantitative bounds on the rate of approach.

Additional questions of whether practical 1-layer transformers can perform and learn dictionary lookup are addressed
in Sections \ref{sec:temperature-simulations} and \ref{section:empirical-results}.

\section{Simulating the Theoretical Construction with Temperature}
\label{sec:temperature-simulations}

The theoretical results for the same position task in Section \ref{sec:same_position_1_layer_construction} assume average hard attention, a common technique for proving transformer expressiveness.  \cite{yang2024simulatinghardattentionusing} show that such proofs for hard attention carry over to transformers with soft attention if a temperature parameter is included in the softmax. We experimentally determine the temperature required for the theoretical construction to achieve 100\% accuracy when using soft attention instead of average hard attention.

\subsection{Embeddings with Multiple Frequencies}

Section \ref{sec:same_position_1_layer_construction} introduces a vector representation of positive integers in $[m]$ computed using a vector of $\ell$ frequencies $f\in [m]^\ell$.
In the proofs in Section \ref{sec:same_position_1_layer_construction}, we limit ourselves to the case of using the single frequency $f=1$. Here, we consider the more general case of using multiple frequencies and the effect of the choice of the frequency vector on the temperature required to achieve 100\% accuracy on the dictionary lookup task using the construction we introduce in Section \ref{sec:same_position_1_layer_construction}.
\footnote{We note that the original definition of the transformer architecture by \cite{Vaswani+2017} uses a position encoding based on multiple frequencies.}

Given a vector of frequencies $f\in [m]^\ell$, $j\in [m]$ and $i\in[\ell]$, 
let $\theta_{m,f_i}(j)$ be the angle $2\pi f_i j / m$, and define $\cs_{m,f_i}(j) = [\cos(\theta_{m,f_i}(j)),\sin(\theta_{m,f_i}(j))]^\top$. The representation of $j$ using frequencies $f_0,f_1,\ldots, f_{\ell-1}$ is then the vector below. 
\[
    \cs_{m,f}(j) 
    = 
    \begin{bmatrix}
    \cs_{m,f_0}(j) \\
    \cs_{m,f_1}(j) \\
    \vdots \\
    \cs_{m,f_{\ell-1}}(j)
    \end{bmatrix}\]

Given a dictionary problem
$(k_0,v_0),(k_1,v_1),\ldots,(k_{n-1},v_{n-1}),k^*$,
we can now construct the same position input embedding $\mathbf{X} = [\mathbf{x}_0, \mathbf{x}_1, \dots, \mathbf{x}_{n-1},\mathbf{x}_n]$ with each $\mathbf{x}_i\in\R^{4\ell}$, as follows.
For $j \in [n]$:
\[
    \mathbf{x}_j 
    = 
    \begin{bmatrix}
    \cs_{N,f}(\idx(k_j)) \\
    \cs_{N,f}(\idx(v_j))
    \end{bmatrix}, \text{ and }
    \mathbf{x}_n
    = \begin{bmatrix}
    \cs_{N,f}(\idx(k^*)) \\
    0_{2\ell \times 1}
\end{bmatrix}\]

We define the matrices $\qproj$, $\kproj$, and $\vproj$, relative to $f$ as:
\[
    \qproj = \kproj = \begin{bmatrix}
    \mat{I}_{2\ell} & \mat{0}_{2\ell \times 2\ell}
    \end{bmatrix}, \qquad
    \vproj = \begin{bmatrix}
    \mat{0}_{2\ell \times 2\ell} & \mat{0}_{2\ell \times 2\ell} \\
    \mat{0}_{2\ell \times 2\ell} & 2 * \mat{I}_{2\ell} 
    \end{bmatrix},
\]

Including additional frequencies that increase the minimum gap between the dot product of a matching pair and that of any distinct pair increases the corresponding attention-score gap.
This has the effect of magnifying the attention weight assigned to the target key and diminishing the attention weights assigned to distractor keys with similar representations to the target key. This in turn reduces noise in the weighted sum of values, increasing the relative contribution of the correct value to the output of attention.

When representing values from $[m]$ using a single frequency, the choice $f=1$ guarantees that no two distinct values are mapped to the same representation, and achieves the smallest possible maximum dot product of representations of two distinct values. In the case of multiple frequencies, it is not straightforward to determine which combination of frequencies minimizes the maximum dot product for representations of distinct values. We determined the best combinations of 2 and 4 frequencies for representing $[1000]$ using exhaustive search. In the case of 8 frequencies, we chose the best candidate among 100 million randomly sampled frequency combinations, since exhaustive search became infeasible at this scale.

For $[1000]$, using only the frequency 1 results in a margin of $0.000010$ between the attention score of a correct match and the score of a query and the key with the most similar embedding. Using the ideal frequency pair $[1,88]$ instead increases that margin by almost three orders of magnitude, to $0.007698$.

For each vector of frequencies we constructed models with the respective frequency embeddings and the weight matrices introduced in the construction above. The models used residuals but no layer normalization. We then performed inference (no training) on randomly generated dictionaries of size 100 and a choice of 1000 keys and values. We gradually decreased the softmax temperature and recorded the accuracy.

\subsection{Results}

As shown in Figure~\ref{fig:temperature-accuracy-layernorm}, we found that with sufficiently low temperatures, all models achieved 100\% accuracy. Models using more frequencies and dimensions for their embeddings succeeded with temperatures closer to 1. Experiments with a slightly modified embedding yielded similar results for models with layer normalization.

\begin{figure}[bt]
  \centering
  \includegraphics[width=1\textwidth]{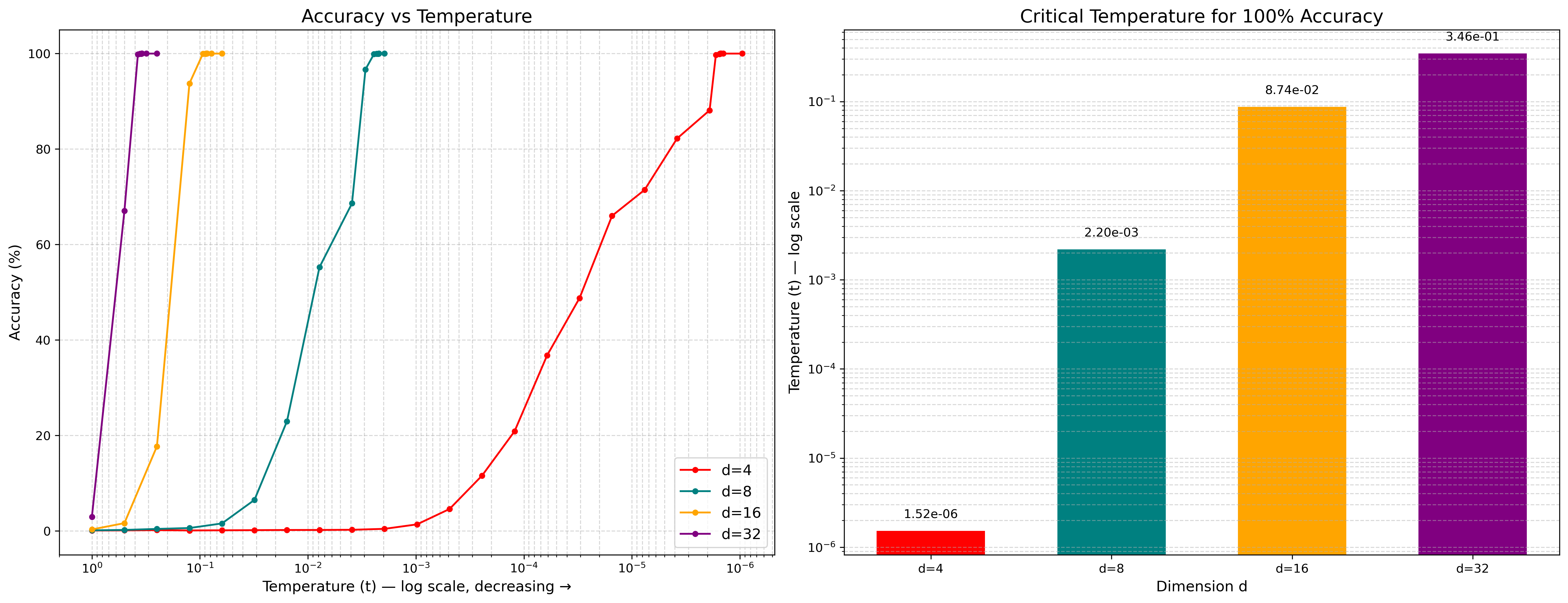}
  \caption{Temperatures required for models with dimensions 4, 8, 16, 32 to achieve 100\% accuracy using the theoretical construction. The embeddings are based on 1, 2, 4, and 8 frequencies, respectively.}
  \label{fig:temperature-accuracy-layernorm}
\end{figure}

\section{Consecutive Positions: Theoretical Bounds}
\label{sec:consecutive_positions_1_layer_bounds}

When the keys $k_i$ and values $v_i$ are stored in consecutive positions, information about a key and its corresponding value requires coordination between specific pairs of different input positions.
It is clear that there must be a position embedding to enable such coordination.
In this section we prove asymptotically matching lower and upper bounds of $\Theta(N \log N)$ on the c-size of a family of 1-layer transformers to perform dictionary lookup for a vocabulary of cardinality $N$ in the consecutive positions case.

\subsection{\texorpdfstring{Consecutive Positions: A 1-Layer Lower Bound of $\Omega(N \log N)$}{Consecutive Positions Case: 1-Layer Lower Bound of Omega(N log N)}}
\label{subsub-worst-case-lower-bound}

\citet{sanford2024onelayertransformersfailsolve} use communication complexity to prove a worst-case lower bound on the c-size of any family of 1-layer transformers to perform the induction heads task.
They prove their lower bound under the assumption that the transformer ``uses $p$ bits of precision'', stated as follows: ``... the outputs of all embedding functions ... and self-attention heads may be rounded to rational numbers with at most $p$ bits of precision without changing the behavior of the mapping [defined by the transformer].''
We refer to this assumption as the \emph{$p$-bit rational precision assumption}.
We adapt their argument to prove the following theorem.
\begin{theorem}
    \label{thm:sanford-lower-bound}
    Suppose $T$ is a 1-layer transformer satisfying the $p$-bit rational precision assumption, with $h$ self-attention heads and embedding dimension $d$.
    If ~$T$ correctly performs the dictionary lookup task for a vocabulary $\vocab$ of cardinality $N$ in the consecutive positions case, then $hdp \ge (1/3)(N\log N)$.
\end{theorem}
Note that this result applies to any 1-layer transformer satisfying the assumptions, including those with arbitrary position encodings and softmax attention or average hard attention. In particular, the claim about arbitrary position encodings is restricted to encodings whose outputs satisfy the $p$-bit rational precision assumption.
The proof is in \cref{appendix_proof_of_sanford_lower_bound}.

\subsection{\texorpdfstring{Consecutive Positions: A 1-Layer Softmax Attention Upper Bound of $O(N \log N)$}{Consecutive Positions: A 1-Layer Softmax Attention Upper Bound of O(N log N)}}
\label{subsec:nlogn_upper_bound}

The general construction for memorization given by \cite{kajitsuka2024aretransformers} implies that there is a transformer with one softmax attention layer that performs the dictionary lookup task correctly, but the upper bound on its c-size derived from their general result is at least $N^N$.
We prove the following stronger result in the specific case of dictionary lookup.

\begin{theorem}
\label{thm:one-layer-consecutive-permuted-large}
For any positive integer $N$, there exists a soft attention transformer with a suitable position embedding that correctly performs dictionary lookup for a vocabulary of cardinality $N$ in the consecutive positions case.
The transformer has one attention layer with softmax attention, one head, embedding dimension $2N+2$, $O(\log N)$ bits of precision, and a constant number of position-wise feedforward layers with a total of $N^{O(1)}$ parameters.
The c-size is $hdp = O(N \log N)$.
\end{theorem}

\begin{proof}
    Let $d = 2N+2$ be the embedding dimension.
    Suppose the input sequence of tokens representing a dictionary lookup problem is $k_0,v_0,\ldots,k_{n-1},v_{n-1},k^*$ for some $n \le N$.
    Let $r_0, r_1, \ldots, r_{2n}$ be the sequence of indices of these tokens, so $r_0 = \idx(k_0)$, $r_1 = \idx(v_0)$ and so on, concluding with $r_{2n} = \idx(k^*)$.
    We describe a construction for a 1-layer transformer that copies a representation of this sequence of indices into the first $2n+1$ dimensions of its activation vector, so that the position-wise feedforward layer at position $2n$ has access to the entire input sequence.
    The position embedding maps position $i \in [2n+1]$ to the one-hot vector in $\R^d$ that has a single $1$ at index $i$.
    The word embedding maps $\tau \in \vocab$ to the vector $[0,0,\ldots,0,\ln(\idx(\tau)+1)]^{\top} \in \R^d$.

    We set $\dkey=1$, $\kproj=[0,\ldots,0,1]$, and $\qproj=[1,\ldots,1,0]$.
    Then for all positions $i$ and $j$, $\kproj\mat{x}_j=\ln(r_j+1)$, $\qproj\mat{x}_i=1$.
    Considering the position of the target key, $2n$, its attention score for position $j$ is  $\ln(r_j + 1)$, which exponentiates to $r_j+1$. Thus the attention weight is $(r_j + 1)/S$, where $S = \sum_{k=0}^{2n} (r_k + 1)$.
    The value matrix extracts the first $d-1$ coordinates of $\mat{x}_j$ and sets the last coordinate to $0$.
    The result of the attention computation at position $2n$ is $[(r_0 + 1)/S, (r_1 + 1)/S, \ldots, (r_{2n} + 1)/S, 0, \ldots, 0]^{\top} \in \R^d$.
    Having a representation of the whole input sequence, a constant number of position-wise feedforward layers with $N^{O(1)}$ parameters and
    $O(\log N)$ bits of precision suffice for the computation.
    The c-size is $hdp = O(N \log N)$.
\end{proof}

\subsection{Consecutive Positions: A Concise Family of 2-Layer Transformers}
\label{sec:consecutive_positions_2_layer_construction}
If we allow our transformers to have \emph{two} layers, there is a concise family of transformers that performs dictionary lookup in the consecutive positions case.

\begin{theorem}
\label{thm:two-layer-consecutive-permuted}
For any positive integer $N$, there exists a 2-layer transformer with a suitable position embedding that performs dictionary lookup for a vocabulary $\vocab$ of cardinality $N$ in the consecutive positions case.
The transformer has embedding dimension $7$, average hard attention, one head, $O(\log N)$ bits of precision, residuals, feedforward layers, and no layernorm. The c-size is $hdp = O(\log N)$.
\end{theorem}

The construction is similar to that for induction heads: using the position embedding, the key information is copied to the same position as the value information, and the 1-layer solution is adapted to produce the correct answer.  The proof is in \cref{appendix:proof_two_layer_consecutive_permuted}.

\section{Training Experiments}
\label{section:empirical-results}

The preceding theoretical results concern expressivity, which is a necessary but not sufficient condition for trainability.
Here we test the ability of small transformers with 1 layer to be trained to perform dictionary lookup in the same position and consecutive positions cases.
These experiments are designed to provide a sanity check for the qualitative distinction between the two cases.

\paragraph{Experimental Setup:}
In all experiments, we use $n$ distinct keys and $n$ possibly non-distinct corresponding values, both taken from $[N]$.
We test cases where the token embedding is learned, as well as cases where we initialize models with custom token embeddings taken from our constructions.
For each input sequence, the keys are a random sequence of $n$ distinct elements of $[N]$  and the target key is a random element of that sequence.
The inputs are presented to the transformer in two principal ways:

\paragraph{Same position cases:} Each input position simultaneously stores a key-value pair by concatenating their token embeddings into a single vector in $\R^d$.
The transformer thus receives both the key and its corresponding value information at each position.

\paragraph{Consecutive position cases:} Keys and their corresponding values are embedded in $\R^d$ and stored in adjacent but separate input positions.
Specifically, keys occupy even-indexed positions, while their associated values occupy subsequent odd-indexed positions.
In each case, $\dhid = d$.

Encoder-only transformers with one attention layer 
were trained for varying embedding dimensions and dictionary sizes, using a single attention head.
All models were trained using the Adam optimizer (learning rate $0.001$), with $10000$ batches of $256$ samples of randomly generated inputs across five independent runs with different random seeds.
A run stops early if its loss reaches $< 0.01$.
We measure performance as accuracy in predicting the correct dictionary lookup for each input across $100$ test batches. Residual connections are applied consistently across all attention and feed-forward layers.
We did not use layer normalization or dropout.

\subsection{Results}
\label{subsec:Results}


\paragraph{Same Position:}
We trained models of dimensions $d=4$, $d=8$, $d=16$, and $d=32$ with both random initialization and using the custom frequency-based embeddings and weight matrices introduced in \cref{sec:temperature-simulations}.
For the custom initialization, $d=4,8,16,32$ correspond to $1,2,4,8$ frequencies, respectively. The construction in \cref{thm:same_position_1_layer} uses only $d=4$; the larger custom models test whether additional frequencies improve trainability, while the larger randomly initialized models test whether additional embedding dimensions improve trainability.
Though $d = 4$ is sufficient for theoretical expressivity, all trained models of this size performed approximately at chance. For all other  embedding dimensions, models initialized with custom weights and embeddings achieved above 99\% accuracy after training. Randomly initialized models struggled in the $d=8$ case but matched the custom models for d=16 and d=32. Results for dictionaries of size 100 and vocabulary of size 1000 are shown in \cref{tab:case_1_to_4_accuracy}.


\begin{table}[ht]
\setlength{\tabcolsep}{4pt}
\centering
\begin{tabular}{c|cc|cc|cc|cc}
\hline
  Initialization & \multicolumn{2}{c|}{$d=4$} & \multicolumn{2}{c|}{$d=8$} & \multicolumn{2}{c|}{$d=16$} & \multicolumn{2}{c}{$d=32$} \\
\hline
random & 1.25 & 0.37 & 61.73 & 1.05 & 99.77 & 0.04 & 99.94 & 0.03 \\
custom & 1.84 & 1.46 & 99.86 & 0.10 & 99.99 & 0.00 & 99.99 & 0.02 \\
\hline
\end{tabular}
\caption{Accuracy (mean and SD) for same position dictionaries with $n=100,N=1000$, using one attention layer.  There is no position embedding.}
\label{tab:case_1_to_4_accuracy}
\end{table}

\begin{figure}[tb]
    \centering
    \includegraphics[width=.9\textwidth]{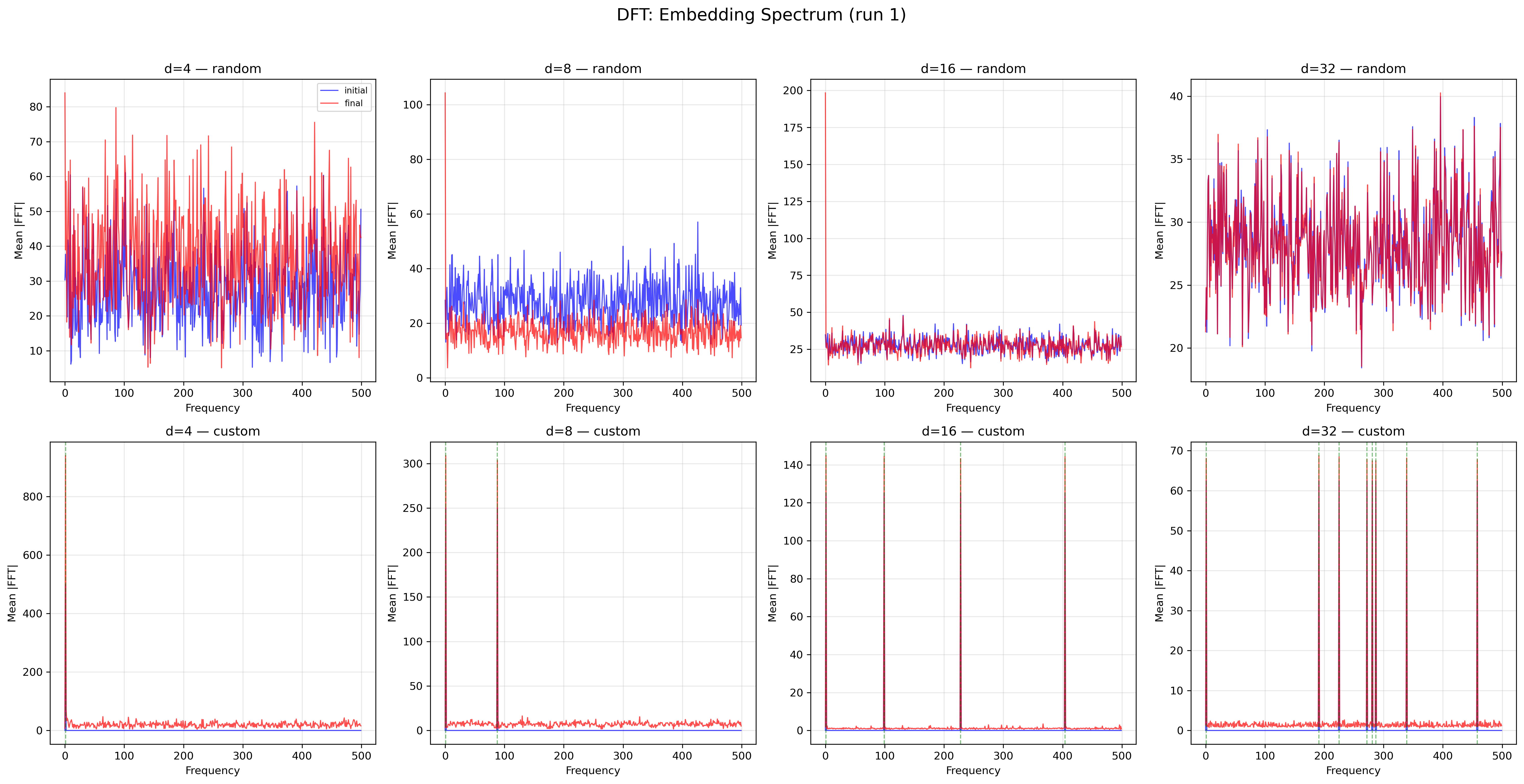}
    \caption{Discrete Fourier transform analysis of initialized (blue) and learned (red) embeddings. The top row corresponds to models with random initialization with different embedding dimension, while the bottom row represents models with custom initializations.  Note the dramatic spike at frequency 1 in the learned embeddings with random initialization for $d=8, 16$.}
    \label{fig:dft}
\end{figure}
We explored the degree to which learned models converged to a representation that makes use of frequency-based encodings. Using the discrete Fourier transform \citep{cormen2001} to analyze the set of embeddings  revealed that models initialized using custom embeddings with varying numbers of frequencies hew closely to these representations throughout training (Table \ref{fig:dft} bottom). In contrast, randomly initialized models for $d=8,16$ come to have clear peaks at frequency 1 but no others, (Table \ref{fig:dft} top). At dimension 32 the randomly initialized model shows no heightened sensitivity to one particular frequency over the others, suggesting that at that point the embedding space has become large enough for other strategies to come online.

\paragraph{Consecutive Positions:}
\cref{tab:accuracy_case_5} shows results for consecutive positions, with each row corresponding to a different dictionary and vocabulary size, $n$ and $N$ respectively.
Recall that, under the $p$-bit rational precision assumption, no concise family of 1-layer transformers performs this task for all $N$. 
We find that 1-layer transformers exhibited significant struggles in learning dictionary lookup, especially for larger $n$ and $N$. In the case of $n=10, N=20$ larger $d$ were associated with higher accuracy. 


\begin{table}[bht]
\centering
\begin{tabular}{cc|cc|cc|cc|cc}
\hline
 $n$  &  $N$  & \multicolumn{2}{c|}{$d=4$} & \multicolumn{2}{c|}{$d=8$} & \multicolumn{2}{c|}{$d=16$} & \multicolumn{2}{c}{$d=32$} \\
\hline
10 & 20 & 13.64 & 0.30 & 18.07 & 0.55 & 21.25 & 0.17 & 21.71 & 0.09 \\
20 & 200 & 3.97 & 0.30 & 3.43 & 0.25 & 3.62 & 0.12 & 4.92 & 0.16 \\
100 & 1000 & 0.45 & 0.04 & 0.57 & 0.04 & 0.58 & 0.03 & 0.47 & 0.03 \\
\hline
\end{tabular}
\caption{Accuracy (mean and SD) for consecutive position dictionaries, across varying $n$ and $N$, using one attention layer. The position embedding is learned.}
\label{tab:accuracy_case_5}
\end{table}

\section{Conclusion}

We have considered the capacity of different abstract models of transformers to perform two variants of a dictionary lookup task, a fundamental building block of many algorithms. In the first variant, keys and values are represented in a shared input position, while in the second variant, each value is stored in the position following that of the associated key. 

We have shown that the first variant of this task can be computed by concise 1-layer transformers; we provide a construction for a family of concise 1-layer transformers that accomplishes the task using average hard attention, one head, embedding dimension 4, no position embedding and $O(\log N)$ bits of precision, where $N$ is the cardinality of the key and value vocabulary. Our construction  uses no residuals, layer norms, or feedforward layers, but these can be incorporated using techniques described by \citet{yang2025transformercookbook}. 
To extend this result to  transformers that use softmax attention, we explore the efficacy of adding a temperature parameter to the softmax computation. We experimentally determine the temperature value necessary to succeed on the task with a frequency-based embedding, which varies by model dimensionality. 

For the second variant of the dictionary lookup task,  where keys and values are stored in consecutive positions, we prove that families of concise 1-layer transformers are not able to perform the task under the assumption of $p$-bit rational precision. We demonstrate however that a family of concise 2-layer transformers can perform this task,  using constant embedding dimension $7$, average hard attention, one head, residuals, feedforward layers, no layernorm, and $O(\log N)$ bits of precision, where $N$ is the cardinality of the key and value vocabulary.

Having the capacity to perform a computation does not, of course, ensure that such a task can be feasibly learned. We therefore briefly explore the viability of learning in 1-layer transformers. We find that in the first variant, transformers are indeed able to achieve ceiling performance, so long as the model dimension is sufficiently large. Attempts to train 1-layer transformers on the consecutive position dictionary lookup task do not succeed, in accordance with the theoretical predictions. 

Returning to the first variant of the task, we find some benefit in initializing the model with the custom  embeddings we define in our construction that use frequency-based encodings: in this case. models with smaller dimensionality ($d=8$) can achieve success. We then  analyzed the learned representations of custom and randomly initialized models  using the discrete Fourier transform. Custom initialized retain their preference for the use of frequency-based representations across learning. In contrast, randomly initialized transformers learn embeddings with a significant peak at frequency 1, but no other frequencies. This suggests that inductive biases of these models, when randomly initialized, do not favor the kind of frequency-based representation we hypothesize and which permits more efficient generalization. We leave it for future work to shed light on the nature of the representations that are preferred in this context, and how to imbue transformer models with a bias in favor of these representations.






% ...


\section{Acknowledgments}
This work includes and builds on results from our earlier, unpublished work ~\citep{strobl2025conciseonelayertransformersfunction}.  The authors thank David Chiang, Alexander Kozachinskiy, Binghui Peng, Clayton Sanford, and Andy Yang for their generous help. We are grateful to members of the Computational 
Linguistics at Yale (CLAY) lab for their useful commentary. We thank the Yale Center for Research Computing for guidance and assistance in computation run on the Bouchet cluster. Lena Strobl was supported by the Knut and Alice Wallenberg Foundation (KAW 2024.0443).
\section*{AI Statement}
AI coding assistance was used to write code for this paper, including for the generation of figures and tables. All outputs were reviewed and verified by the authors.

\bibliography{main}
\bibliographystyle{colm2025_conference}

\pagebreak
\appendix

\section{\texorpdfstring{Example of the Representation of Integers from $[m]$}{Example of the Representation of Integers from [m]}}

\label{circle_graphic}

\begin{figure}[ht]
  \centering
  \includegraphics[width=0.6\textwidth]{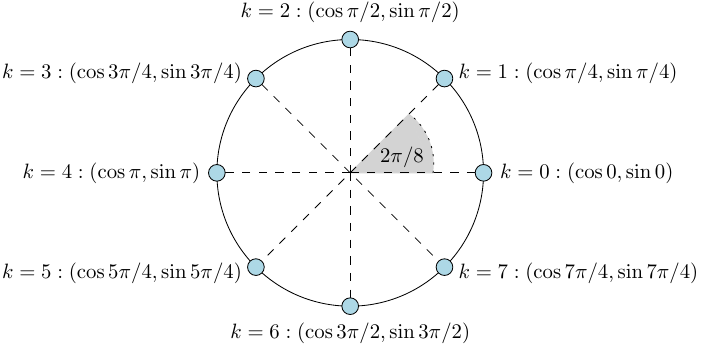}
  \caption{Each integer $k \in [8]$ is mapped to a point on the unit circle using the transformation $\cs_8(k) = [\cos(\frac{2\pi k}{8}), \sin(\frac{2\pi k}{8})]$.
  E.g., $k=0$ maps to $(1,0)$, $k=2$ maps to $(0,1)$, and so on.
  Dashed radial lines show the angles between consecutive points.}
  \label{fig:sincosconstruction}
\end{figure}

\section{\texorpdfstring{Example of the Construction in \cref{thm:same_position_1_layer}}{Example of the 1-Layer Construction for Same Position}}
\label{appendix:example_of_same_position_1_layer}

\begin{example}[Example for \cref{thm:same_position_1_layer}]
In \cref{thm:same_position_1_layer}, we prove that for any positive integer $N$, there exists a 1-layer transformer that performs dictionary lookup for a vocabulary of cardinality $N$ in the same position case.
The transformer has average hard attention, one head, an embedding dimension of 4, and $O(\log N)$ bits of precision.
Here we describe an example of the construction for $N=4$ and a dictionary of length $4$.

\textbf{Representation of Integers.}  
Each integer $ k\in [4] $ is encoded using the mapping
\[
    \cs_4(0) = \begin{bmatrix} 1 \\ 0 \end{bmatrix},\quad
    \cs_4(1) = \begin{bmatrix} 0 \\ 1 \end{bmatrix},\quad
    \cs_4(2) = \begin{bmatrix} -1 \\ 0 \end{bmatrix},\quad
    \cs_4(3) = \begin{bmatrix} 0 \\ -1 \end{bmatrix}.
\]
This encoding corresponds to mapping the integer $k$ to a point on the unit circle with angle $\theta_4(k)=\frac{2\pi k}{4}$.

\textbf{Input Construction.}  
Consider a dictionary consisting of the following key–value pairs:
\[
    (k_0,v_0) = (0,2),\quad
    (k_1,v_1) = (2,3),\quad
    (k_2,v_2) = (1,2),\quad
    (k_3,v_3) = (3,1).
\]
Let the target key be $ k^*=2 $, so the correct output is $3$.
The input sequence is constructed as follows.
For $ j\in\{0,1,2,3\} $, define
\[
    \mat{x}_j = \begin{bmatrix} \cs_4(k_j) \\ \cs_4(v_j) \end{bmatrix}.
\]
Thus,
\[
    \mat{x}_0 = \begin{bmatrix} \cs_4(0) \\ \cs_4(2) \end{bmatrix} 
    = \begin{bmatrix} 1 \\ 0 \\ -1 \\ 0 \end{bmatrix},\enspace
    \mat{x}_1 = \begin{bmatrix} \cs_4(2) \\ \cs_4(3) \end{bmatrix} 
    = \begin{bmatrix} -1 \\ 0 \\ 0 \\ -1 \end{bmatrix},\enspace
    \mat{x}_2 = \begin{bmatrix} \cs_4(1) \\ \cs_4(2) \end{bmatrix} 
    = \begin{bmatrix} 0 \\ 1 \\ -1 \\ 0 \end{bmatrix},\enspace
    \mat{x}_3 = \begin{bmatrix} \cs_4(3) \\ \cs_4(1) \end{bmatrix} 
    = \begin{bmatrix} 0 \\ -1 \\ 0 \\ 1 \end{bmatrix}.
\]
An additional input vector encodes the target key:
\[
    \mathbf{x}_4 = \begin{bmatrix} \cs_4(i^*) \\ \mathbf{0}_{2\times1} \end{bmatrix}
    = \begin{bmatrix} \cs_4(2) \\ 0 \\ 0 \end{bmatrix}
    = \begin{bmatrix} -1 \\ 0 \\ 0 \\ 0 \end{bmatrix}.
\]

\textbf{Linear Transformations.}  
The transformer uses the following weight matrices:
\[
    \qproj = \begin{bmatrix} \mat{I}_2 & \mat{0} \end{bmatrix},\quad
    \kproj = \begin{bmatrix} \mat{I}_2 & \mat{0} \end{bmatrix},\quad
    \vproj = \begin{bmatrix} \mat{0}_{2\times 2} & \mat{0}_{2\times 2}\\ \mat{0}_{2\times 2} & 2*\mat{I}_2 \end{bmatrix}.
\]
Thus, the query at the target position (position 4) is computed as
\[
    \mat{q}_4 = \qproj \mat{x}_4 
    = \begin{bmatrix} -1 \\ 0 \end{bmatrix},
\]
and for positions $0$ through $4$, the key vectors are
\[
    \mat{k}_0 = \begin{bmatrix} 1 \\ 0 \end{bmatrix},\quad
    \mat{k}_1 = \begin{bmatrix} -1 \\ 0 \end{bmatrix},\quad
    \mat{k}_2 = \begin{bmatrix} 0 \\ 1 \end{bmatrix},\quad
    \mat{k}_3 = \begin{bmatrix} 0 \\ -1 \end{bmatrix},\quad
    \mat{k}_4 = \begin{bmatrix} -1 \\ 0 \end{bmatrix},
\]
with corresponding value vectors
\[
    \mat{v}_0 = \begin{bmatrix} 0 \\ 0 \\ -2 \\ 0 \end{bmatrix},\quad
    \mat{v}_1 = \begin{bmatrix} 0 \\ 0 \\ 0 \\ -2 \end{bmatrix},\quad
    \mat{v}_2 = \begin{bmatrix} 0 \\ 0 \\ -2 \\ 0 \end{bmatrix},\quad
    \mat{v}_3 = \begin{bmatrix} 0 \\ 0 \\ 0 \\ 2 
    \end{bmatrix},\quad
    \mat{v}_4 = \begin{bmatrix} 0 \\ 0 \\ 0 \\ 0
    \end{bmatrix}.
\]

\textbf{Attention Score Computation.}  
The attention score between the query at position 4 and a key at position $j$ is given by
\[
    s_j = \frac{q_4^\top k_j}{\sqrt{2}},
\]
which yields:
\[
    s_0 = \frac{-1}{\sqrt{2}}, \quad
    s_1 = \frac{1}{\sqrt{2}},  \quad
    s_2 = \frac{0}{\sqrt{2}}, \quad
    s_3 = \frac{0}{\sqrt{2}}, \quad
    s_4 = \frac{1}{\sqrt{2}}.
\]

\textbf{Average Hard Attention and Output.}  
Using average hard attention, the transformer assigns attention weight $1/2$ to the two positions with the maximum score of $\frac{1}{\sqrt{2}}$ and $0$ to the other positions, that is

\[
    \alpha_1=  \alpha_4 = \frac{1}{2} \quad \text{and} \quad \alpha_j=0 \text{ for } j \in \{0,2,3\}.
\]
Thus, the output of the attention layer is
\[
    \mat{h} = \frac{1}{2}\mat{v}_1 + \frac{1}{2}\mat{v}_4 = \begin{bmatrix} 0 \\ 0 \\ 0 \\ -1 \end{bmatrix}.
\]

\textbf{Unembedding.}  
Finally, an unembedding matrix $\mat{U}$ is used to convert $\mat{h}$ into the final output.
We have the following.
\[  \hat{\mat{v}}
    = 
    \mat{U} \mat{h} 
    =
    \begin{bmatrix}
    0 & 0 & \cos 0 & \sin 0 \\[1mm]
    0 & 0 & \cos \frac{\pi}{2} & \sin \frac{\pi}{2}\\[1mm]
    0 & 0 & \cos \pi & \sin \pi\\[1mm]
    0 & 0 & \cos \frac{3\pi}{2} & \sin \frac{3\pi}{2}
    \end{bmatrix}
    \begin{bmatrix} 0 \\ 0 \\ 0 \\ -1 \end{bmatrix}
    =
    \begin{bmatrix}
    0 & 0 & 1 & 0 \\[1mm]
    0 & 0 & 0 & 1 \\[1mm]
    0 & 0 & -1 & 0\\[1mm]
    0 & 0 & 0 & -1
    \end{bmatrix}
    \begin{bmatrix} 0 \\ 0 \\ 0 \\ -1 \end{bmatrix}
    = 
    \begin{bmatrix} 0 \\ -1 \\ 0 \\ 1 \end{bmatrix}.
\]
Because the maximum entry in $\hat{\mat{v}}$ is at index 3, the transformer outputs $3$, which matches the correct dictionary lookup output of $3$.
\end{example}

\section{\texorpdfstring{Proof of \cref{thm:sanford-lower-bound}}{Proof of 1-layer lower bound in the consecutive positions case}}
\label{appendix_proof_of_sanford_lower_bound}

We define the \emph{function evaluation task with domain $[N]$}, which is a specialization of the dictionary lookup task such that the vocabulary is $[N]$ and every dictionary has length $N$.
For such inputs, the dictionary can be viewed as representing a finite function $f:[N] \rightarrow [N]$, and the problem of dictionary lookup can be viewed as evaluating $f$ on a particular input $k^*$.
Note that there are $N^N$ such functions.

We consider the following two-party communication complexity problem, which we term \emph{evaluate $f$ with domain $[N]$ on $0$.}
Let $k_0,\ldots,k_{N-1}$ be a permutation of $[N]$ and let $v_0,\ldots,v_{N-1}$ be any sequence of elements from $[N]$, where these two sequences represent the function $f(k_i) = v_i$.
Alice knows the sequence $v_0,\ldots,v_{N-1}$ and Bob knows the sequence $k_0,\ldots,k_{N-1}$.
Alice sends a message to Bob, after which he is required to output the unique $v_j$ such that $k_j = 0$.

\begin{lemma}
    \label{lemma:communication-complexity-worst-case}
    In any communication protocol for the problem of evaluating $f$ with domain $[N]$ on $0$ in which Bob always outputs the correct answer, Alice's message must sometimes contain at least $N \log_2 N$ bits.
\end{lemma}

\begin{proof}
    If not, there will be two different sequences $v_0,\ldots,v_{N-1}$ and $v_0',\ldots,v_{N-1}'$ for which Alice sends the same message.
    There is some position $j$ such that $v_j \neq v_j'$.
    Thus, when $k_j = 0$, Bob cannot produce the correct (different) answers for both sequences.
\end{proof}

\citet{sanford2024onelayertransformersfailsolve} give a careful analysis of bit precision in a 1-layer transformer computation that shows the following.
\begin{lemma}
    \label{lemma:transformer-to-protocol}
    Suppose $T$ is a 1-layer transformer satisfying the $p$-bit rational precision assumption, with $h$ self-attention heads and embedding dimension $d$.
    If\, $T$  correctly performs the function evaluation task with domain $[N]$ in the consecutive positions case, then there is a communication protocol for the problem of evaluating $f$ with domain $[N]$ on $0$ in which the message sent by Alice has length at most $3hdp$.
\end{lemma}

\cref{thm:sanford-lower-bound} follows from these two lemmas because if a
transformer $T$ correctly performs the dictionary lookup task for a vocabulary of cardinality $N$ in the consecutive positions case, then $T$ also performs the function evaluation task with domain $[N]$ in the consecutive positions case.

\section{\texorpdfstring{Proof of \cref{thm:two-layer-consecutive-permuted}}{Proof of Theorem for Two-Layer Construction}}
\label{appendix:proof_two_layer_consecutive_permuted}

\begin{proof}
The vocabulary is $\vocab$ with $N = |\vocab|$.
The sequence of inputs is
$k_0,v_0,\ldots,k_{n-1},v_{n-1},k^*$.
We use two layers with average hard attention.
The first layer copies $k_r$ from position $2r$ into position $2r+1$, so that position $2r+1$ contains both $k_r$ and $v_r$.
The second layer applies a version of the 1-layer construction from \cref{thm:same_position_1_layer} for the same position case.
It ignores the even positions and extracts $v_r$ such that $k_r = k^*$.

The token and position embeddings are defined as follows, for $\tau \in \vocab$ and $r \in [n+1]$.
\begin{align*}
    \mathbf{w}(\tau) &= [0,0,\cs_{N}(\idx(\tau)),0,0,0]^\top \\
    \mathbf{p}(2r) &= [0,0,0,0,\cs_{N+1}(r),1]^\top \\
    \mathbf{p}(2r+1) &= [0,0,0,0,\cs_{N+1}(r),0]^\top.
\end{align*}
Note that the position embedding for $k_r$ and $v_r$ are the same in coordinates $4$ and $5$ but different in coordinate $6$.
The token embedding and the position embedding use different representations of integers, $\cs_{N}$ and $\cs_{N+1}$.
Thus the input vectors $\mat{x}_i$ are as follows.

If $r \in [n]$ and $i = 2r$ then
\[
\mat{x}_i = [0,0,\cs_{N}(\idx(k_r)),\cs_{N+1}(r),1]^\top.
\]
If $r \in [n]$ and $i = 2r+1$ then
\[
\mat{x}_i = [0,0,\cs_{N}(\idx(v_r)),\cs_{N+1}(r),0]^\top.
\]
For $i = 2n$,
\[
\mat{x}_i = [0,0,\cs_{N}(\idx(k^*)),\cs_{N+1}(n),1]^\top.
\]

\paragraph{Layer 1.}
Matrices $\qproj_1, \kproj_1 \in \R^{3 \times 7}$ both map $\mat{x}_i$ to the vector $[\cs_{N+1}(r),0]^\top$, where $i = 2r$ or $i = 2r+1$, by extracting coordinates $4$ and $5$ of $\mat{x}_i$.
For $r \in [n]$, the attention scores for positions $i = 2r$ and $i = 2r+1$ are thus maximized at positions $j = 2r$ and $j = 2r+1$, each of which gets attention weight $1/2$.
The attention score for position $i = 2n$ is uniquely maximized at position $j = 2n$, which gets attention weight $1$.
The matrix $\vproj_1 \in \R^{7 \times 7}$ copies the two values from dimensions $2$ and $3$ of the attended-to positions $j$ into the dimensions $0$ and $1$ of the result, and sets the other dimensions to $0$.
Thus the results $\mat{y}_i$ of the first attention layer after the residual is added are as follows.

If $r \in [n]$ and $i = 2r$ then
\[
\mat{y}_i = [\frac{1}{2}\cs_{N}(\idx(k_r))+ \frac{1}{2}\cs_{N}(\idx(v_r)),\cs_{N}(\idx(k_r)),\cs_{N+1}(r),1]^\top.
\]
If $r \in [n]$ and $i = 2r+1$ then
\[
\mat{y}_i = [\frac{1}{2}\cs_{N}(\idx(k_r))+ \frac{1}{2}\cs_{N}(\idx(v_r)),\cs_{N}(\idx(v_r)),\cs_{N+1}(r),0]^\top.
\]
For $i = 2n$,
\[
\mat{y}_i = [\cs_{N}(\idx(k^*)),\cs_{N}(\idx(k^*)),\cs_{N+1}(n),1]^\top.
\]
The first feedforward layer subtracts the values in coordinates $2$ and $3$ from twice the values in coordinates $0$ and $1$, Thus the inputs $\mat{x}'_i$ to the second attention layer are as follows.

If $r \in [n]$ and $i = 2r$ then
\[
\mat{x}'_i = [\cs_{N}(\idx(v_r)),\cs_{N}(\idx(k_r)),\cs_{N+1}(r),1]^\top.
\]
If $r \in [n]$ and $i = 2r+1$ then
\[
\mat{x}'_i = [\cs_{N}(\idx(k_r)),\cs_{N}(\idx(v_r)),\cs_{N+1}(r),0]^\top.
\]
For $i = 2n$,
\[
\mat{x}'_i = [\cs_{N}(\idx(k^*)),\cs_{N}(\idx(k^*)),\cs_{N+1}(n),1]^\top.
\]

\paragraph{Layer 2.}
We use the approach from \cref{thm:same_position_1_layer} for the same position case, making use of the even/odd portion of the position embedding.
The matrix $\kproj_2 \in \R^{3 \times 7}$ copies dimensions $0$, $1$, and $6$ from $\mat{x}'_i$.
Thus, for $r \in [n]$ and position $j = 2r$, the key vector is
\[
\mat{k}_{j} =
\begin{bmatrix}
    \cs_{N}(\idx(v_r)) \\
    1
\end{bmatrix},
\]
for $r \in [n]$ and position $j = 2r+1$, the key vector is
\[
\mat{k}_{j} =
\begin{bmatrix}
    \cs_{N}(\idx(k_r)) \\
    0
\end{bmatrix},
\]
and for $j = 2n$, the key vector is
\[
\mathbf{k}_{2n} =
\begin{bmatrix}
    \cs_{N}(\idx(k^*)) \\
    1
\end{bmatrix}.
\]
The matrix $\qproj_2 \in \R^{3 \times 7}$ copies dimensions $0$, $1$ of $\mathbf{x}_i'$ and multiplies dimension $6$ by $-1$.
Thus, for position $2n$, the query is
\[
\mat{q}_{2n} =
\begin{bmatrix}
    \cs_{N}(\idx(k^*)) \\
    -1
\end{bmatrix}.
\]
The attention score for position $2n$ is uniquely maximized at $j^* = 2r+1$ such that $k_r = k^*$.
The matrix $\vproj_2$ copies dimensions $2$ and $3$ of the attended-to position $j^*$ and sets the other dimensions to $0$, yielding the desired value $[\cs_{N}(\idx(v_r)),0,0,0,0,0]^\top$.
The unembedding matrix $\mat{U}$ consists of $N$ rows of the form $[\cs_{N}(i),0,0,0,0,0]$ for $i \in [N]$, which gives the correct output $v_r$.
\end{proof}

\end{document}